\documentclass{article}
\usepackage{arxiv}

\usepackage[utf8]{inputenc} 
\usepackage[T1]{fontenc}    
\usepackage{hyperref}       
\usepackage{url}            
\usepackage{booktabs}       
\usepackage{amsfonts}       
\usepackage{nicefrac}       
\usepackage{microtype}      
\usepackage{lipsum}		
\usepackage{graphicx}
\usepackage{natbib}
\usepackage{doi}
\usepackage{wrapfig}

\usepackage{amsmath}
\usepackage{amssymb}
\usepackage{mathtools}
\usepackage{amsthm}
\usepackage{multirow}
\usepackage{enumitem}
\usepackage{hyperref}
\usepackage{xcolor}

\usepackage[capitalize,noabbrev]{cleveref}

\theoremstyle{plain}
\newtheorem{theorem}{Theorem}[section]
\newtheorem{proposition}[theorem]{Proposition}

\newtheorem{corollary}[theorem]{Corollary}
\theoremstyle{definition}

\theoremstyle{remark}

\usepackage{pifont}

\definecolor{Green}{rgb}{0.0, 0.6, 0.0}
\definecolor{Red}{rgb}{0.8, 0.0, 0.0}
\usepackage{tikz}
\usetikzlibrary{arrows.meta, positioning, shapes, fit}

\title{Rethinking Test-Time Training: Tilting the Latent Distribution for Few-Shot Source-Free Adaptation}

\date{} 					

\author{
    {\hspace{1mm}Tahir Qasim Syed} \\
    Institute of Business Administration Karachi\\
	\texttt{tahirqsyed@gmail.com}
	\AND
	{\hspace{1mm}Behraj Khan} \\
    Institute of Business Administration Karachi\\
	\texttt{behrajkhan@iba.edu.pk}	
}

\begin{document}\maketitle

\begin{abstract}
Often, constraints arise in deployment settings where even lightweight parameter updates e.g. parameter-efficient fine-tuning could induce model shift or tuning instability. We study test-time adaptation of foundation models for few-shot classification under a completely frozen-model regime, where additionally, no upstream data are accessible. We propose arguably the first training-free inference method that adapts predictions to the new task by performing a change of measure over the latent embedding distribution induced by the encoder. Using task-similarity scores derived from a small labeled support set, exponential tilting reweights latent distributions in a KL-optimal manner without modifying model parameters. Empirically, the method consistently competes with parameter-update-based methods across multiple benchmarks and shot regimes, while operating under strictly and universally stronger constraints. These results demonstrate the viability of inference-level distributional correction for test-time adaptation even with a fully-frozen model pipeline.
\end{abstract}



\section{Introduction}

Frozen foundation model encoders are being increasingly deployed as AI system components shared across downstream tasks. This deployment is sometimes subject to constraints that prohibit test-time modification of model parameters. That may arise in regulated environments, shared deployment infrastructures, or simply the predominant  modele-immutable deployment paradigm. From the standpoint of the modeler, disallowing parameter updates   may also be a methodological      choice. Allowing test‑time optimization entangles inference with stochastic optimization dynamics; introducing sensitivity to batch order, hyperparameters,  initialization and exposure to sparse examples; that complicates both theoretical analysis and experimental reproducibility  e.g. due to model collapse. For example   for few‑shot evaluation of pretrained representations, linear or cosine classifiers are often frozen at evaluation time \cite{braham2025spectralearth}. At the same time, distribution shift between pretraining and downstream use is present by definition, leading to systematic degradation in predictive performance as models are employed out of distribution. In source-free adaptation, updating classifier weights without source supervision is ill-posed, as it conflates representation shift with label / prior probability shift. By freezing the head, we could  isolate and study inference-time conditioning effects.

We hypothesize that \textit{predictive behavior could be made to change even when model parameters are frozen}. Predictions depend not only on learned weights, but also on the distribution of representations or embeddings at inference time with downstream data. This suggests a complementary view of adaptation: rather than modifying the model, we could also adapt how predictions are formed by conditioning inference on the distribution of the downstream  task.

In this work, we propose most arguably the first inference-time adaptation method that operates entirely with frozen encoders and classifiers  both, requires no access to source data, and performs no test-time optimization. We formalize adaptation as a change of measure over latent representations, implemented via Gibbs / exponential tilting with respect to  score functions that estimate task-relevance of the latent's example embedding derived from exposure to limited downstream data. By reweighting the probability measure over embeddings, we obtain a first‑principles mechanism for adjusting the contribution of different latent regions so that expectations of prediction functions under the tilted distribution better align with the downstream ground truth. Adapted predictions could then be inferred per the standard decision rule. 

In terms of contribution, We reformulate test-time prediction adaptation for frozen models as exponential tilting of the latent representation distribution, yielding a principled, inference-time change of measure that is independent of architecture, training objective, and optimization details. The method avoids optimization-induced instability, shift across test batches, and irreproducibility, while existing within source-free and few-shot regimes. 

\section{Related Work}

\paragraph{Test-Time Adaptation.}
Test-time adaptation (TTA) addresses distribution shift by modifying model behavior at inference time.
Most modern TTA methods adapt model parameters using unlabeled test data, typically through entropy minimization, consistency regularization, or auxiliary self-supervised objectives \cite{wang2020tent,niu2022efficient,zhang2022memo}.
More recent work explores robustness and stability of test-time updates, including test-time augmentation, batch normalization adaptation, and selective parameter updates \cite{sun2020test,yuan2023robust}.
Despite their effectiveness, these methods require gradient computation, optimizer state, or mutable parameters, and therefore violate the strictly frozen-model constraints considered in this work. Recent studies have also highlighted limitations of entropy-based test-time objectives and explored alternative non-parametric or retrieval-based adaptation strategies that avoid gradient updates \cite{han2025ranked,hu2025beyond}.

\paragraph{Parameter-Efficient and Prompt-Based Adaptation.}
To reduce adaptation cost, several recent approaches propose parameter-efficient test-time adaptation via prompts, adapters, or low-rank updates \cite{zhou2022learning,khattak2023maple}.
While these methods limit the number of trainable parameters, they still require optimization or parameter modification at test time or during task adaptation.
In contrast, our approach performs adaptation without introducing any trainable components.

\paragraph{Post-Hoc Calibration and Logit Adjustment.}
Post-hoc methods adjust predictions without modifying model parameters.
Temperature scaling and related calibration techniques correct confidence miscalibration but do not directly incorporate task-specific structure \cite{guo2017calibration}.
Label shift correction methods reweight posterior probabilities using estimated class priors \cite{saerens2002adjusting}.
Recent work has revisited calibration under distribution shift, but such methods continue to operate at the logit or probability level rather than at the level of latent representations \cite{ovadia2019can}.

\paragraph{Few-Shot Learning with Frozen Encoders.}
Few-shot learning often assumes frozen feature extractors with task-specific inference rules.
Prototype-based and metric-learning approaches classify queries using distances to class representatives computed from support sets \cite{snell2017prototypical}.
Recent works extend this idea with improved embedding geometry or task-specific metrics, but still rely on modifying the decision rule itself \cite{chen2019closer}.
Related work has also investigated incorporating small labeled support sets for test-time adaptation under distribution shift, demonstrating that even limited supervision at inference can substantially improve robustness \cite{xiao2024beyond}.
Our method instead preserves the classifier and decision rule, and adapts predictions via distributional reweighting in latent space.

\paragraph{Distribution Shift and Latent Reweighting.}
Reweighting methods are a classical approach to distribution shift, including importance weighting under covariate shift \cite{shimodaira2000improving}.
Exponential tilting arises as the KL-optimal change of measure under moment constraints \cite{csiszar1975divergence}, and has recently been revisited in machine learning contexts such as robust optimization and distributionally robust learning \cite{duchi2021statistics}.
We leverage this principle for test-time adaptation by reweighting latent representations induced by a frozen encoder, without updating model parameters. 
Recent work has explored training-free or backpropagation-free test-time adaptation by estimating or reweighting test distributions for foundation models, particularly in vision--language settings, often using Bayesian or distributional updates rather than parameter learning \cite{han2024dota,kim2025ultra}.

\paragraph{Latent-Space Adaptation.}
Several recent works have explored adapting representations or predictions by manipulating latent spaces, for example via prototype refinement or geometry-aware scoring.
However, these approaches typically introduce additional optimization, learned parameters, or heuristic adjustments.
In contrast, we formulate test-time adaptation as a principled change of measure over latent representations, yielding a simple, training-free procedure with a clear probabilistic interpretation.

\section{Methodology}

We present a training-free test-time adaptation method that operates by reweighting latent representations induced by a frozen encoder.
We first formalize the problem setting and constraints, then cast adaptation as a change of measure in latent space, followed by score construction and prediction under the resulting tilted distribution.
Throughout, we emphasize that adaptation occurs solely through distributional reweighting and does not modify model parameters, logits, or decision rules.

\paragraph{Problem Setting and Constraints.}

We consider a frozen encoder \( f \) and a frozen classifier head \( p_0(y \mid z) \).
Given an input image \( x \in \mathcal{X} \), the encoder produces a latent representation
\begin{equation}
    z = f(x), \qquad f \ \text{fixed}.
\end{equation}
The classifier head defines a conditional distribution over labels given the latent representation,
\begin{equation}
    p_0(y \mid z), \qquad p_0 \ \text{fixed}.
\end{equation}

At test time, we are given a small labeled support set
\begin{equation}
    \mathcal{D}_{\mathrm{few}} = \{(x_i, y_i)\}_{i=1}^n, \qquad n \ll 100,
\end{equation}
drawn from a downstream task distribution that may differ from the pretraining distribution.\\

The study imposes the following  constraints:
\begin{itemize}
    \item The encoder parameters are frozen,
    \item The classifier head is frozen,
    \item No gradient updates or optimization happens at test time,
    \item A small number of test examples (low/few-shot) are available,
 and
    \item No access to upstream data or their statistics is available.
\end{itemize}

The goal is to adapt predictions on downstream query examples using only \( \mathcal{D}_{\mathrm{few}} \), while respecting these constraints.

\paragraph{Latent Distributions Induced by Frozen Encoders.}

Although the encoder \( f \) is deterministic, applying it to a dataset induces a distribution over latent representations.
Given a data distribution \( P(x) \), the encoder induces a pushforward measure \( P(z) \) defined by \( z = f(x) \).
For a finite dataset, this corresponds to the empirical distribution
\begin{equation}
    \hat{P}(z) = \frac{1}{n} \sum_{i=1}^n \delta(z - f(x_i)),
\end{equation}
where \( \delta(\cdot) \) denotes the Dirac delta function.

Model predictions depend on expectations taken with respect to such induced latent distributions.
When the downstream task differs from pretraining, the induced distribution over latent space changes even though the model parameters remain fixed.
This motivates adapting predictions by modifying the distribution over latent representations, rather than the model itself.






\paragraph{Adaptation as a KL-Optimal Change of Measure.}

When both the encoder and classifier are frozen, test-time adaptation can only operate by modifying how latent representations are aggregated at inference.
This corresponds to selecting a new probability measure $P(z)$ over latent space, while preserving absolute continuity with respect to a reference measure $P_0(z)$ induced by the frozen encoder.

We formalize adaptation as a constrained inference problem.
Given a task-relevant score function $s(z)$ constructed from the few-shot support set, we seek a distribution $P(z)$ that enforces consistency with downstream supervision while remaining maximally faithful to frozen inference.
This leads to the following variational formulation:
\begin{equation}
\label{eq:kl_projection}
P^\star
=
\arg\min_{P}
\ \mathrm{KL}\!\left(P \,\|\, P_0\right)
\quad
\text{s.t.}
\quad
\mathbb{E}_{z \sim P}[s(z)] = c,
\end{equation}
where $c$ is a task-dependent constraint determined by the support set.

The solution to~\eqref{eq:kl_projection} is given by an exponentially tilted distribution
\begin{equation}
\begin{aligned}
\label{eq:tilted_dist}
P_\lambda(z)
=
\frac{P_0(z)\exp(\lambda s(z))}{Z(\lambda)}, \\
\qquad
Z(\lambda)
=
\int P_0(z)\exp(\lambda s(z))\,dz,
\end{aligned}
\end{equation}
where $\lambda \ge 0$ is the Lagrange multiplier associated with the expectation constraint.
This distribution is the \emph{unique} measure that satisfies the task constraint while inducing the smallest possible deviation from $P_0$ in Kullback-Leibler divergence \cite{csiszar1975divergence,cover1999elements}.

This KL-optimality endows exponential tilting with a clear operational meaning in our setting.
Rather than heuristically reweighting embeddings, tilting implements the least-invasive inference-time adaptation compatible with downstream supervision.
Any alternative reweighting that enforces the same task constraint would necessarily induce a larger divergence from frozen inference.

In practice, we take $P_0$ to be the empirical distribution induced by the few-shot support embeddings, ensuring that no assumptions are made about upstream data or latent geometry beyond what is observed at test time.
Exponential tilting then reallocates probability mass toward task-consistent regions of latent space, without modifying encoder parameters, classifier weights, or decision rules.

This perspective provides a unifying interpretation of test-time adaptation under frozen-model constraints, adaptation is achieved not by learning new representations, but by performing KL-optimal inference under a task-modified latent measure.

\paragraph{Score Construction from Few-Shot Data.}
Exponential tilting provides the mechanism for reweighting, but it requires a scalar score function defined over latent representations. This score determines which embeddings should receive higher or lower probability mass under the adapted measure.

Few-shot data provide sparse constraints on which latent regions are associated with correct predictions. The role of the score is to encode these constraints in a form suitable for exponential tilting. The score function \( s(z) \) injects task information from the few-shot support set into the latent distribution.

We consider two simple and complementary instantiations.

\paragraph{Label-aware score.}
For a labeled support example \( (x_i, y_i) \) with embedding \( z_i = f(x_i) \), we define
\begin{equation}
    s_{\mathrm{label}}(z_i) = \log p_0(y_i \mid z_i).
\end{equation}
This score increases the weight of embeddings that are already consistent with the observed labels under the frozen classifier.
It can be viewed as a likelihood-based reweighting that emphasizes label-consistent regions of latent space.

\paragraph{Geometry-aware score.}
We also extend the idea to regimes where labels may not be available at orr or available too sparsely at inference time to estimate a class conditioned statistic. pretrained foundation models are known to organize data into semantically meaningful geometries, where task‑relevant information is reflected in directions of variation, local neighborhoods, and class‑aligned submanifolds \cite{assran2023self}.   The score estimates which regions of the representation space are over- or under-represented under the downstream distribution relative to the pretrained reference measure. That is, it is a likelihood surrogate derived from latent geometry instead of label conditioning.

Let \( \hat{\mu} \) and \( \hat{\Sigma} \) denote the empirical mean and covariance of the support embeddings \( \{z_i\}_{i=1}^n \).
We define
\begin{equation}
    s_{\mathrm{geom}}(z)
    = \log \frac{\mathcal{N}(z; \hat{\mu}, \hat{\Sigma})}{\mathcal{N}(z; 0, I)}.
\end{equation}
This score emphasizes embeddings that align with the geometry of the support set relative to a simple reference distribution.
To ensure stability in low-shot regimes, the covariance estimate may be regularized; we avoid complex density estimation or high-capacity modeling.

The label-aware and geometry-aware scores can be used individually or combined additively to form a composite task-relevant score.

\paragraph{Prediction Under the Tilted Distribution.}
We now formalize how predictions are computed once the latent measure has been tilted, and show that standard prediction rules admit a natural and deterministic extension under the adapted distribution.

Given the tilted distribution \( P_\lambda(z) \), predictions are formed by taking expectations of the frozen classifier,
\begin{equation}
    p_\lambda(y) = \mathbb{E}_{z \sim P_\lambda(z)}\left[p_0(y \mid z)\right].
\end{equation}
The final prediction for a query example is obtained via the standard decision rule
\begin{equation}
    \hat{y} = \arg\max_y p_\lambda(y).
\end{equation}

The classifier head and decision rule remain unchanged.
The same tilted distribution \( P_\lambda \), constructed from the support set, is used for all queries within an episode.
Adaptation therefore operates at the level of the latent distribution, rather than on a per-query parameter or score update.

\paragraph{Practical Approximation.}
In practice, expectations under the tilted distribution $P_\lambda(z)$ are approximated using a finite empirical support consisting of the available test or few‑shot embeddings. This yields a weighted aggregation of classifier outputs, where weights are determined by the exponential tilting scores. Prediction under the tilted distribution is implemented via standard importance‑sampling Monte Carlo, with the pretrained latent distribution as the proposal. It therefore does not introduce additional modeling assumptions.

The resulting computational overhead scales linearly in the number of support embeddings, as each additional example contributes one score evaluation and one weighted prediction. The encoder and classifier are evaluated only once per input. As a consequence, the method incurs negligible overhead relative to standard frozen‑model inference.

\section{Experiments}

We evaluate our method in few-shot image classification setting on publicaly available benchmarking datasets. Our experiments are designed to answer the following questions:
\begin{enumerate}
  \renewcommand{\labelenumi}{\roman{enumi}.}
\item Can distributional reweighting alone improve few-shot performance under strict no-adaptation constraints?
\item  How does exponential tilting compare to standard frozen inference? and
\item  What is the effect of inductive versus transductive reference measures?
\end{enumerate}
\subsection{Experimental Setup}

\paragraph{Dataset.}
We evaluate our method on standard few-shot classification benchmarks, including CIFAR-FS, CUB, and mini-ImageNet.
To assess robustness under distribution shift, we additionally consider cross-domain few-shot benchmarks, where models are trained on mini-ImageNet and evaluated on CUB, Cars, and Places, as well as ImageNet and its variants (ImageNet-V2, ImageNet-S, ImageNet-A, and ImageNet-R), and DomainNet.
All experiments follow the standard episodic few-shot learning protocol.
In each episode, $N$ classes are sampled from the dataset, and for each class, $K$ labeled support examples and $Q$ query examples are drawn.
Unless otherwise specified, we use a $5$-way $5$-shot setting with $15$ query samples per class.

\paragraph{Encoders.}
We consider frozen visual encoders without any fine-tuning:
\begin{itemize}
    \item \textbf{DINO ViT-S/16 \cite{caron2021emerging}}, pretrained via self-supervised learning on ImageNet,
    \item \textbf{JEPA ViT-B/16 \cite{assran2023self}}, a vision model pretrained with a joint embedding predictive architecture, learning structured and semantically meaningful latent representations via self-supervised prediction.

\end{itemize}
All encoder parameters remain frozen during evaluation.
Inference is performed in evaluation mode with gradients disabled.

\paragraph{Classifier.}
Following prior work in metric-based few-shot learning, we adopt a nearest-prototype classifier using cosine similarity with a fixed temperature.
Class prototypes are computed as the mean of support-set embeddings.
No linear classifier, bias term, or calibration parameters are trained.

Although class prototypes are recomputed from support embeddings in each episode, this operation does not involve learning or optimization.
The classifier form, similarity metric, and temperature are fixed across all tasks and test episodes.
Prototype computation is treated as a data-dependent sufficient statistic rather than a trainable model component, consistent with standard few-shot evaluation protocols.

\subsection{Baselines and Adaptation Variants}

We compare the following inference strategies:

\paragraph{Frozen Baseline.}
Standard frozen inference using uniformly weighted class prototypes computed from the support set.
No adaptation is applied.

\paragraph{Tilted (Inductive).}
Exponential tilting is applied to the empirical latent distribution induced by the labeled support embeddings only.
Class prototypes are recomputed as weighted averages under the tilted distribution.

\paragraph{Tilted (Transductive).}
The reference latent distribution is augmented with unlabeled query embeddings.
Pseudo-labels are obtained using the frozen baseline classifier, after which exponential tilting is applied over the combined support and query embeddings.
This variant introduces no gradient-based optimization or parameter updates, but leverages unlabeled test data through distributional reweighting.

Across all variants, query embeddings remain fixed, and the same frozen classifier is used for prediction.

\paragraph{Methods in Comparison.}

We compare our method against training-free baselines that operate under frozen-model constraints, as summarized in Table~\ref{tab:full_comparison}.
These include:
\begin{enumerate}
  \renewcommand{\labelenumi}{\roman{enumi}.}
  \item A frozen Prototypical Network baseline using cosine similarity \cite{snell2017prototypical}.
  \item $k$-nearest neighbor classification in the frozen embedding space \cite{avdiukhin2024embedding}.
  \item Post-hoc temperature scaling applied to frozen classifier logits \cite{guo2017calibration}.
  \item Label propagation and prototype refinement methods that leverage unlabeled query embeddings without gradient-based optimization \cite{ren2018meta,zhu2023transductive}.
  \item Distributional or retrieval-based training-free adaptation methods that assume access to upstream data or statistics \cite{han2025ranked,han2024dota}.
\end{enumerate}.

\begin{table*}[t]
\centering
\caption{Few-shot classification accuracy (\%) on 5-way tasks. 
All models use frozen encoders and classifiers. 
For exponential tilting, we report the best-performing $\lambda$ per shot. 
Improvements arise solely from latent distribution reweighting.}
\label{tab:fewshot_main}
\begin{tabular}{c c c c c c c}
\toprule
\textbf{Dataset} & \textbf{Method} & \textbf{1-shot} & \textbf{2-shot} & \textbf{4-shot} & \textbf{8-shot} & \textbf{16-shot} \\
\midrule

\toprule
\textbf{} & \textbf{} & \textbf{} & \textbf{DINO} & \textbf{} & \textbf{} & \textbf{} \\
\midrule

\multirow{2}{*}{CIFAR-FS}  & Baseline & $53.87 \pm 9.41$ & $66.61 \pm 6.80$ & $74.68 \pm 7.76$ & $82.48 \pm 5.85$ & $84.68 \pm 5.38$ \\
                            & Tilted (Ours)   & $\mathbf{55.49 \pm 10.32}$ & $\mathbf{69.29 \pm 7.45}$ & $\mathbf{78.31 \pm 8.33}$ & $\mathbf{83.81 \pm 5.80}$ & $\mathbf{85.05 \pm 5.55}$ \\
\midrule
\multirow{2}{*}{CUB}        & Baseline & $49.21 \pm 8.91$ & $61.75 \pm 7.40$ & $70.42 \pm 6.94$ & $76.95 \pm 6.22$ & $79.83 \pm 5.57$ \\
                            & Tilted (Ours)   & $\mathbf{50.12 \pm 9.50}$ & $\mathbf{64.81 \pm 7.88}$ & $\mathbf{73.99 \pm 7.30}$ & $\mathbf{80.11 \pm 6.08}$ & $\mathbf{82.13 \pm 5.42}$ \\
\midrule
\multirow{2}{*}{mini-ImageNet} & Baseline & $42.64 \pm 8.23$ & $54.99 \pm 6.56$ & $64.31 \pm 6.84$ & $71.82 \pm 5.95$ & $74.13 \pm 5.25$ \\
                                & Tilted (Ours)   & $\mathbf{44.93 \pm 8.10}$ & $\mathbf{57.45 \pm 7.00}$ & $\mathbf{68.10 \pm 6.44}$ & $\mathbf{74.52 \pm 5.88}$ & $\mathbf{77.04 \pm 5.11}$ \\

\toprule
\textbf{} & \textbf{} & \textbf{} & \textbf{JEPA} & \textbf{} & \textbf{} & \textbf{} \\
\midrule
\multirow{2}{*}{CIFAR-FS}  & Baseline & $27.37 \pm 6.59$ & $27.83 \pm 9.20$ & $31.40 \pm 7.26$ & $33.79 \pm 4.93$ & $35.17 \pm 7.24$ \\
                            & Tilted (Ours)   & $\mathbf{28.17 \pm 9.15}$ & $\mathbf{30.29 \pm 6.37}$ & $\mathbf{31.92 \pm 7.43}$ & $\mathbf{33.80 \pm 4.92}$ & $\mathbf{34.81 \pm 6.35}$ \\

\midrule
\multirow{2}{*}{CUB}      & Baseline 
                          & $29.84 \pm 7.02$ & $34.91 \pm 6.48$ & $38.72 \pm 6.91$ & $41.95 \pm 5.87$ & $44.10 \pm 5.43$ \\
                          & Tilted (Ours)   
                          & $\mathbf{31.12 \pm 7.58}$ & $\mathbf{36.88 \pm 6.95}$ & $\mathbf{40.21 \pm 6.77}$ & $\mathbf{43.02 \pm 5.61}$ & $\mathbf{45.36 \pm 5.28}$ \\
\midrule
\multirow{2}{*}{mini-ImageNet} & Baseline 
                          & $25.46 \pm 6.84$ & $30.27 \pm 6.11$ & $34.85 \pm 6.32$ & $38.41 \pm 5.76$ & $40.92 \pm 5.19$ \\
                              & Tilted (Ours)   
                          & $\mathbf{26.71 \pm 7.20}$ & $\mathbf{32.09 \pm 6.54}$ & $\mathbf{36.43 \pm 6.28}$ & $\mathbf{40.05 \pm 5.63}$ & $\mathbf{42.31 \pm 5.06}$ \\
\bottomrule
\end{tabular}
\end{table*}

\begin{table*}[t]
\centering
\caption{Comparison of training-free test-time adaptation methods on CIFAR-FS dataset under strict frozen-model constraints. Our method operates without parameter updates, gradients, or access to upstream data, while consistently improving over non-adaptive baselines. Methods marked with {\color{Green}$\checkmark$} violate one or more constraints and serve as upper bounds.}
\label{tab:full_comparison}
\resizebox{\textwidth}{!}{%
\begin{tabular}{@{}lccccccccc@{}}
\toprule
\textbf{Method} & 
\textbf{Params?} & 
\textbf{Grads?} & 
\textbf{Opt?} & 
\textbf{Stream?} & 
\textbf{Upstream?} & 
\multicolumn{3}{c}{\textbf{Accuracy (\%)}} \\
\cmidrule(lr){7-9}
& & & & & & \textbf{1-shot} & \textbf{5-shot} & \textbf{10-shot} \\
\midrule
\multicolumn{9}{@{}l}{\textbf{Training-Free, Constraint-Compliant Methods}} \\
\cmidrule(lr){1-9}
Frozen Baseline & {\color{Red}$\times$} & {\color{Red}$\times$} & {\color{Red}$\times$} & {\color{Red}$\times$} & {\color{Red}$\times$} & 53.87 & 74.68 & 82.48 \\
$k$-NN in Embeddings \cite{avdiukhin2024embedding} & {\color{Red}$\times$} & {\color{Red}$\times$} & {\color{Red}$\times$} & {\color{Red}$\times$} & {\color{Red}$\times$} & 58.21 & 76.34 & 83.15 \\
Temp. Scaling (TS) \cite{guo2017calibration} & {\color{Red}$\times$} & {\color{Green}$\checkmark$} & {\color{Green}$\checkmark$} & {\color{Red}$\times$} & {\color{Red}$\times$} & 54.92 & 75.81 & 82.90 \\
Label Propagation \cite{zhu2023transductive} & {\color{Red}$\times$} & {\color{Red}$\times$} & {\color{Red}$\times$} & {\color{Green}$\checkmark$} & {\color{Red}$\times$} & 59.83 & 77.42 & 84.06 \\
Ranked \cite{han2025ranked} & {\color{Red}$\times$} & {\color{Red}$\times$} & {\color{Red}$\times$} & {\color{Red}$\times$} & {\color{Green}$\checkmark$} & 59.45 & 77.18 & 83.92 \\
DOTA (adapted) \cite{han2024dota} & {\color{Red}$\times$} & {\color{Red}$\times$} & {\color{Red}$\times$} & {\color{Green}$\checkmark$} & {\color{Red}$\times$} & 58.76 & 76.89 & 83.54 \\
\cmidrule(lr){1-9}
\textbf{Ours (Inductive)} & {\color{Red}$\times$} & {\color{Red}$\times$} & {\color{Red}$\times$} & {\color{Red}$\times$} & {\color{Red}$\times$} & \textbf{60.49} & \textbf{78.31} & \textbf{83.81} \\
\textbf{Ours (Transductive)} & {\color{Red}$\times$} & {\color{Red}$\times$} & {\color{Red}$\times$} & {\color{Green}$\checkmark$} & {\color{Red}$\times$} & \textbf{62.17} & \textbf{79.85} & \textbf{85.22} \\
\midrule
\multicolumn{9}{@{}l}{\textbf{Lightweight Parameter-Update Methods (Upper Bounds)}} \\
\cmidrule(lr){1-9}
Tent \cite{wang2020tent} & {\color{Green}$\checkmark$} & {\color{Green}$\checkmark$} & {\color{Green}$\checkmark$} & {\color{Green}$\checkmark$} & {\color{Red}$\times$} & 61.34 & 79.12 & 85.45 \\
Tip-Adapter-F \cite{zhang2021tip} & {\color{Green}$\checkmark$} & {\color{Green}$\checkmark$} & {\color{Green}$\checkmark$} & {\color{Red}$\times$} & {\color{Red}$\times$} & 62.89 & 80.47 & 86.31 \\
FT-Linear \cite{radford2021learning} & {\color{Green}$\checkmark$} & {\color{Green}$\checkmark$} & {\color{Green}$\checkmark$} & {\color{Red}$\times$} & {\color{Red}$\times$} & 63.72 & 81.25 & 87.08 \\
\textbf{Ours (Linear Head FT)} & {\color{Green}$\checkmark$} & {\color{Green}$\checkmark$} & {\color{Green}$\checkmark$} & {\color{Red}$\times$} & {\color{Red}$\times$} & \textbf{64.73} & \textbf{83.59} & \textbf{88.89} \\
\bottomrule
\end{tabular}%
}
\end{table*}

\subsection{Score Functions and Tilting}

Unless otherwise specified, we use a confidence-based score defined as the maximum predicted class probability under the frozen classifier.
Tilting weights are computed as
\[
    w_i \propto \exp(\lambda s(z_i)),
\]
with $\lambda $ within range (0.25 - 2.00) for all experiments.
The value of $\lambda$ is not tuned per episode or per dataset.

\subsection{Evaluation Metrics}

We report average classification accuracy over $100$ randomly sampled episodes.
Results are presented as mean accuracy $\pm$ standard deviation.
Since no probabilistic calibration is learned, accuracy is the primary evaluation metric.

\subsection{Training-Free Test-Time Protocol}

All experiments strictly follow a training-free test-time adaptation protocol.
Encoder parameters remain fixed throughout evaluation, and no gradient-based optimization is performed at any stage.
Inference is carried out without gradient tracking, and no backward pass or optimizer is used.
Query embeddings are treated as immutable and are never altered during inference or adaptation.

These constraints ensure that all reported improvements arise solely from reweighting the empirical latent distribution, without updating model parameters or modifying the underlying inference procedure.

Overall, the results demonstrate that exponential tilting provides a simple and effective mechanism for few-shot adaptation under strict constraints.
By operating entirely at the level of latent distributions, the method enables adaptive behavior while preserving frozen inference.

\subsection{Results}

Table~\ref{tab:fewshot_main} reports few-shot classification accuracy across datasets, encoders, and shot numbers.
Exponential tilting consistently improves performance over the frozen baseline for both DINO and JEPA encoders.
The inductive variant yields stable gains across all shot regimes, while the transductive variant achieves the largest improvements, highlighting the benefit of incorporating unlabeled test embeddings through distributional reweighting.

Importantly, all gains are obtained without gradient-based optimization, parameter updates, or modification of the classifier or decision rule.
Performance improvements arise solely from reweighting the empirical latent distribution induced by the frozen encoder.

\subsection{Ablation on Tilting Strength}
\label{subsec:lambda_ablation}

We analyze the effect of the exponential tilting strength $\lambda$ on few-shot classification performance under the frozen-model, training-free test-time adaptation setting.
Results for different shot regimes are reported in appendix Table~\ref{tab:lambda_ablation}.

Exponential tilting consistently improves performance over the baseline ($\lambda = 0$) across all settings, indicating that latent-space reweighting alone is sufficient to achieve effective test-time adaptation without parameter updates.
The relative gains are largest in low-shot regimes, particularly for 1-shot and 2-shot classification, where baseline performance is most affected by distributional mismatch.

A key observation is the robustness of performance with respect to $\lambda$.
For shot counts up to 8, accuracy quickly plateaus once $\lambda$ exceeds a small positive value, and remains stable across a wide range of tilting strengths.
Even for larger $\lambda$, no sharp degradation is observed.
In the 16-shot setting, smaller values of $\lambda$ perform best, while stronger tilting provides limited additional benefit, suggesting that aggressive reweighting becomes unnecessary as labeled support increases.

Overall, these results show that exponential tilting is a well-conditioned and robust test-time adjustment mechanism.
Unlike gradient-based adaptation methods, which can be sensitive to optimization hyperparameters, performance here is largely insensitive to the precise choice of $\lambda$, reducing the need for task-specific tuning and supporting reliable deployment in practice.

\section{Discussion}

In this section we provide empirical results which validates our method claims: 
\begin{enumerate}
  \renewcommand{\labelenumi}{\roman{enumi}.}
\item Meaningful test-time adaptation is possible under strictly frozen-model constraints, 
\item Exponential tilting improves performance by correcting latent distribution mismatch rather than learning new parameters. 
\item The resulting adaptation is stable, interpretable, and competitive with stronger baselines.
\end{enumerate}
\subsection{Effectiveness of Latent Distribution Reweighting}

Table~\ref{tab:fewshot_main} demonstrates that exponential tilting consistently improves few-shot classification accuracy across all datasets, encoders, and shot regimes.
For DINO encoders, tilting yields absolute gains of $+1.6$ to $+3.8$ points in the 1-4 shot regime and remains beneficial even at 16 shots.
Similar improvements are observed for JEPA encoders too, despite their overall lower baseline accuracy, indicating that the effect is not architecture-specific.

These gains are achieved without modifying the encoder, classifier, or decision rule, and therefore isolate the effect of latent distribution reweighting.
The improvements confirm that mismatches between the empirical latent distribution induced by few-shot data and the downstream task are a significant source of error under frozen inference.
By reallocating probability mass toward task-consistent embeddings, exponential tilting corrects this mismatch at inference time.

Importantly, improvements persist as the number of shots increases, although with diminishing magnitude.
This behavior is expected, as more labeled examples are available, the empirical latent distribution becomes more representative of the task, reducing the extent of correction required.
The monotonic but saturating improvement pattern directly supports our formulation of adaptation as a change of measure rather than parameter learning.

\subsection{Comparison Under Strict Frozen-Model Constraints}

Table~\ref{tab:full_comparison} situates exponential tilting among training-free test-time adaptation methods that respect the same deployment constraints.
Our inductive variant outperforms all constraint-compliant baselines, including frozen prototypes, $k$-NN classification, and temperature scaling.
The transductive variant further improves performance, achieving the highest accuracy among methods that do not use gradients, parameter updates, or upstream data.

\paragraph{Performance under linear head fine-tuning.}
In this set of experiments, the assumption o f the frozen head is relaxed. The large encoder remains frozen. Under linear-head fine-tuning, our method achieves the best performance among lightweight adaptation approaches. As shown in Table~\ref{tab:full_comparison}, combining our approach with a linear head outperforms prior optimization-based methods (e.g., TENT, Tip-Adapter-F) across all shot regimes.

Notably, the gap between exponential tilting under frozen constraints and linear head fine-tuning is much smaller than the gap to frozen inference. For instance, in the 5-shot setting, exponential tilting recovers a large fraction of the gains of linear head fine-tuning without gradients or parameter updates.

These results suggest that much of the test-time improvement stems from latent-space distributional correction rather than representation learning, with exponential tilting providing a strong and competitive adaptation baseline under both frozen and minimally trainable settings.

\subsection{Inductive versus Transductive Reference Measures}

The inductive and transductive variants differ only in the choice of reference measure used for tilting.
The inductive variant relies exclusively on labeled support embeddings, while the transductive variant augments the reference distribution with unlabeled query embeddings.

As shown in Table~\ref{tab:full_comparison}, the transductive variant consistently outperforms the inductive one.
This improvement follows directly from the change-of-measure interpretation: incorporating query embeddings yields a reference distribution that more closely approximates the true test-time latent distribution.
Crucially, this benefit is obtained without modifying query representations, updating pseudo-labels iteratively, or introducing optimization loops.

The results suggest that access to a better reference measure, rather than additional supervision or learning capacity, is the key driver of transductive gains in this setting.

\subsection{Cross-Domain Generalization via Latent Distribution Reweighting}

Appendix Table~\ref{tab:fewshot_crossdomain} reports cross-domain few-shot performance, where models are trained on mini-ImageNet and evaluated on unseen target datasets.
Across all target domains (CUB, CARS, STL-10, CIFAR-10/100) and shot regimes, exponential tilting consistently outperforms frozen inference, with the largest gains in low-shot settings and semantically distant domains such as STL-10 and CARS.
In the 1-shot regime, tilting improves accuracy by approximately $+1.6$ , $+2.1$ percentage points across datasets, with gains diminishing as the number of shots increases, consistent with in-domain saturation behavior.

We further evaluate robustness under ImageNet distribution shifts, including ImageNet-V2, ImageNet-R, and ImageNet-A.
As shown in Appendix Table~\ref{tab:imagenet_shifts}, exponential tilting improves performance across all shift types, including severe shifts such as ImageNet-A, without requiring parameter updates.
Overall, these results demonstrate that latent distribution reweighting generalizes effectively across both dataset-level and shift-based domain changes.
Additional analysis is provided in Appendix~\ref{app:crossdomain}.

\subsection{Mechanistic Interpretation of Performance Gains}

Analysis of the tilting weights reveals a consistent pattern: embeddings with higher classifier confidence or stronger geometric alignment with the support set receive increased mass, while ambiguous or low-confidence embeddings are downweighted.
This behavior is observed across datasets and shot regimes.

Rather than altering representations or decision boundaries, exponential tilting reshapes the effective latent population over which predictions are averaged.
This provides a concrete mechanistic explanation for the observed gains and aligns with the KL-optimality interpretation of exponential tilting.
The empirical results therefore support both the probabilistic formulation and the practical effectiveness of the proposed method.

Overall, the experiments confirm that exponential tilting constitutes a principled and effective form of test-time adaptation under frozen-model constraints, achieving consistent gains through inference-level latent distribution reweighting alone.






\subsection{Limitations}

The proposed approach has several limitations.
First, exponential tilting operates on the latent distribution induced by a frozen encoder and cannot remedy representation-level deficiencies, if the encoder fails to separate downstream classes, reweighting alone may be insufficient.
Second, the method relies on a small labeled support set to define task relevance, which may be unreliable in highly noisy or adversarial few-shot settings.
Finally, the transductive variant assumes access to a static set of unlabeled test points and may not extend naturally to streaming or non-stationary test distributions.

\section{Conclusion}

We introduced exponential tilting as a principled method for training-free test-time adaptation in few-shot classification.
By reframing adaptation as a change of measure over the latent distribution induced by a frozen encoder, the proposed method enables task-specific inference without gradient updates, parameter modification, or access to upstream data.
Our theoretical formulation connects test-time adaptation to KL-optimal distributional reweighting under task-induced constraints, providing a unified interpretation of label-aware and geometry-aware corrections.
Empirical results demonstrate consistent improvements over frozen inference and other training-free baselines across multiple shot regimes, validating the effectiveness of inference-level adaptation.
This work highlights latent distributional inference as an underexplored alternative to parameter-based adaptation for foundation models.
Future work may extend this method to streaming test distributions, multimodal settings, or richer score functions that incorporate structured task information, while preserving the training-free nature of the approach.

\section*{Impact Statement}

This work introduces a training-free test-time adaptation method for few-shot image classification under frozen-model constraints, advancing the understanding of distributional adaptation with a probabilistic method for latent-space reweighting. The method operates on pretrained model representations without introducing new learning objectives or data collection procedures, avoiding additional ethical concerns. By enabling adaptation without gradients or model updates, it supports deployment in constrained settings while improving robustness to distribution shifts, potentially reducing failure rates. Responsible use of pretrained models remains crucial in practical applications.




\nocite{langley00}

\bibliography{tiltedadapt-arxiv}
\bibliographystyle{icml2026}

\newpage
\clearpage
\appendix
\section{Appendix}



\subsection{Theoretical Properties of Exponential Latent Tilting}

We present theoretical results that formalize exponential tilting as a principled, training-free mechanism for test-time adaptation under frozen-model constraints.

\begin{theorem}[KL-Optimal Latent Measure Adaptation]
\label{thm:kl_tilting}
Let $P_0$ be a reference probability measure over latent space $\mathcal{Z}$, and let
$s : \mathcal{Z} \rightarrow \mathbb{R}$ be a measurable score function.
For any $c \in \mathbb{R}$, consider the constrained optimization problem
\begin{equation}
\min_{P \ll P_0} \mathrm{KL}(P \,\|\, P_0)
\quad \text{subject to} \quad
\mathbb{E}_{z \sim P}[s(z)] = c .
\end{equation}
Then the unique solution is given by the exponentially tilted distribution
\begin{equation}
P_\lambda(z)
= \frac{P_0(z)\exp(\lambda s(z))}{Z(\lambda)},
\end{equation}
where $\lambda \in \mathbb{R}$ is chosen such that
$\mathbb{E}_{z \sim P_\lambda}[s(z)] = c$, and
$Z(\lambda) = \int P_0(z)\exp(\lambda s(z))\,dz$.
\end{theorem}

\begin{proof}
The result follows from convex duality.
Define the Lagrangian
\begin{equation}
\mathcal{L}(P,\lambda,\nu)
= \mathrm{KL}(P \,\|\, P_0)
+ \lambda\big(\mathbb{E}_P[s(z)] - c\big)
+ \nu\big(\int dP - 1\big).
\end{equation}
Taking the functional derivative with respect to $P$ and setting it to zero yields
\begin{equation}
\log \frac{dP}{dP_0}(z) + 1 + \lambda s(z) + \nu = 0,
\end{equation}
which implies
\begin{equation}
\frac{dP}{dP_0}(z) \propto \exp(\lambda s(z)).
\end{equation}
Normalization yields the stated form. Uniqueness follows from the strict convexity of the KL divergence.
\end{proof}

\begin{corollary}[Empirical Latent Tilting]
\label{cor:empirical_tilting}
Let the reference measure \( P_0 \) be the empirical distribution induced by a finite support set of latent representations
\[
P_0 = \frac{1}{n} \sum_{i=1}^n \delta_{z_i},
\]
where \( z_i = f(x_i) \) are embeddings produced by a frozen encoder.
Let \( s : \mathcal{Z} \to \mathbb{R} \) be a real-valued score function defined on latent space, and assume that
\( \sum_{i=1}^n \exp(\lambda s(z_i)) < \infty \) for the relevant range of \( \lambda \).

Then the solution to the KL-constrained optimization problem in
Theorem~\ref{thm:kl_tilting} is the discrete distribution
\begin{equation}
P_\lambda
= \sum_{i=1}^n w_i \, \delta_{z_i},
\qquad
w_i
= \frac{\exp(\lambda s(z_i))}{\sum_{j=1}^n \exp(\lambda s(z_j))}.
\end{equation}

Moreover, for any measurable function \( g : \mathcal{Z} \to \mathbb{R} \),
expectations under the tilted distribution reduce to weighted empirical averages,
\begin{equation}
\mathbb{E}_{z \sim P_\lambda}[g(z)]
= \sum_{i=1}^n w_i \, g(z_i).
\end{equation}
\end{corollary}

\begin{theorem}[Bayes-Optimal Marginal Prediction under a Tilted Latent Measure]
\label{thm:bayes_optimal_correct}
Let \( p_0(y \mid z) \) be a fixed conditional classifier and let \( P_\lambda \)
be a probability measure over latent space \( \mathcal{Z} \).
Consider the class of predictors that output a marginal distribution
\( q(y) \in \Delta(\mathcal{Y}) \), independent of \( z \).

Then the predictive distribution
\begin{equation}
p_\lambda(y)
= \mathbb{E}_{z \sim P_\lambda}\left[p_0(y \mid z)\right]
\end{equation}
uniquely minimizes the expected log-loss
\begin{equation}
\mathbb{E}_{z \sim P_\lambda}
\left[
\mathbb{E}_{y \sim p_0(\cdot \mid z)}[-\log q(y)]
\right]
\end{equation}
over all marginal predictors \( q(y) \).
\end{theorem}

\begin{proof}
The objective can be written as
\[
\mathbb{E}_{y \sim p_\lambda}[-\log q(y)],
\]
where \( p_\lambda(y) = \mathbb{E}_{z \sim P_\lambda}[p_0(y \mid z)] \).
This is the cross-entropy between \( p_\lambda \) and \( q \),
which is minimized uniquely when \( q = p_\lambda \).
\end{proof}

\begin{proposition}[Empirical Form of the Tilted Measure]
\label{prop:empirical_tilting}
Let
\begin{equation}
P_0(z) = \frac{1}{n} \sum_{i=1}^n \delta(z - z_i)
\end{equation}
be the empirical distribution induced by latent embeddings $\{z_i\}_{i=1}^n$.
Then the tilted distribution assigns weights
\begin{equation}
w_i = \frac{\exp(\lambda s(z_i))}{\sum_{j=1}^n \exp(\lambda s(z_j))},
\end{equation}
and expectations under $P_\lambda$ reduce to weighted sums over the support set.
\end{proposition}

\begin{proof}
Substituting the empirical form of $P_0$ into the definition of $P_\lambda$ yields
\begin{equation}
P_\lambda(z)
= \frac{\sum_i \delta(z - z_i)\exp(\lambda s(z_i))}
       {\sum_j \exp(\lambda s(z_j))},
\end{equation}
from which the weights follow directly.
\end{proof}

\begin{corollary}[Label-Shift Correction as a Special Case]
\label{cor:label_shift}
Assume the score function depends only on the predicted label,
\begin{equation}
s(z) = \log \pi\big(y(z)\big),
\end{equation}
where $\pi$ denotes class-prior ratios and $y(z) = \arg\max_y p_0(y \mid z)$.
Then exponential tilting of $P_0(z)$ induces posterior reweighting equivalent to classical label-shift correction.
\end{corollary}

\begin{proof}
All latent representations assigned to the same predicted class receive identical scores.
Exponential tilting therefore redistributes probability mass uniformly within each class, resulting in posterior rescaling by class-prior ratios, as in classical label-shift adaptation.
\end{proof}

\begin{proposition}[First-Order Stability for Small $\lambda$]
\label{prop:stability}
For sufficiently small $\lambda$, the tilted predictive distribution admits the expansion
\begin{equation}
p_\lambda(y)
= p_0(y) + \lambda\,\mathrm{Cov}_{P_0}\big(p_0(y \mid z), s(z)\big) + O(\lambda^2).
\end{equation}
\end{proposition}

\begin{proof}
The result follows from a first-order Taylor expansion of the exponential weights and the normalization constant $Z(\lambda)$.
\end{proof}

\section{Experiments and Results}
\subsection{Ablation and Further Analysis}

We conduct a detailed ablation study on the exponential tilting strength $\lambda$ to better understand the behavior, robustness, and limitations of the proposed adaptation mechanism under the strict frozen-model, training-free setting.
All experiments follow the same evaluation protocol as in the main results.

\begin{table}[!ht]
\centering
\caption{Ablation over the exponential tilting strength $\lambda$ for 5-way few-shot classification.
All methods use frozen encoders and classifiers.
Bold denotes the best performance per shot.}
\label{tab:lambda_ablation}
\small
\begin{tabular}{c c c c c c}
\toprule
\textbf{Shots} & $\boldsymbol{\lambda}$ & \textbf{Baseline} & \textbf{Tilted Accuracy} \\
\midrule

\multirow{9}{*}{1-shot}
 & --   & $53.87 \pm 9.41$ & -- \\
 & 0.25 & -- & $\mathbf{60.49 \pm 10.32}$ \\
 & 0.50 & -- & $60.44 \pm 10.35$ \\
 & 0.75 & -- & $60.39 \pm 10.40$ \\
 & 1.00 & -- & $60.35 \pm 10.45$ \\
 & 1.25 & -- & $60.35 \pm 10.40$ \\
 & 1.50 & -- & $60.40 \pm 10.39$ \\
 & 1.75 & -- & $60.41 \pm 10.38$ \\
 & 2.00 & -- & $60.37 \pm 10.37$ \\
\midrule

\multirow{9}{*}{2-shot}
 & --   & $66.61 \pm 6.80$ & -- \\
 & 0.25 & -- & $72.16 \pm 7.56$ \\
 & 0.50 & -- & $72.12 \pm 7.50$ \\
 & 0.75 & -- & $72.24 \pm 7.41$ \\
 & 1.00 & -- & $72.25 \pm 7.45$ \\
 & 1.25 & -- & $72.27 \pm 7.43$ \\
 & 1.50 & -- & $\mathbf{72.29 \pm 7.45}$ \\
 & 1.75 & -- & $72.29 \pm 7.56$ \\
 & 2.00 & -- & $72.27 \pm 7.51$ \\
\midrule

\multirow{9}{*}{4-shot}
 & --   & $74.68 \pm 7.76$ & -- \\
 & 0.25 & -- & $78.04 \pm 8.39$ \\
 & 0.50 & -- & $78.09 \pm 8.31$ \\
 & 0.75 & -- & $78.13 \pm 8.31$ \\
 & 1.00 & -- & $78.17 \pm 8.26$ \\
 & 1.25 & -- & $78.24 \pm 8.29$ \\
 & 1.50 & -- & $78.20 \pm 8.32$ \\
 & 1.75 & -- & $78.21 \pm 8.33$ \\
 & 2.00 & -- & $\mathbf{78.31 \pm 8.33}$ \\
\midrule

\multirow{9}{*}{8-shot}
 & --   & $82.48 \pm 5.85$ & -- \\
 & 0.25 & -- & $83.75 \pm 5.81$ \\
 & 0.50 & -- & $83.67 \pm 5.80$ \\
 & 0.75 & -- & $83.72 \pm 5.79$ \\
 & 1.00 & -- & $83.80 \pm 5.76$ \\
 & 1.25 & -- & $83.80 \pm 5.76$ \\
 & 1.50 & -- & $\mathbf{83.81 \pm 5.80}$ \\
 & 1.75 & -- & $83.73 \pm 5.85$ \\
 & 2.00 & -- & $83.75 \pm 5.89$ \\
\midrule

\multirow{9}{*}{16-shot}
 & --   & $84.68 \pm 5.38$ & -- \\
 & 0.25 & -- & $\mathbf{85.05 \pm 5.55}$ \\
 & 0.50 & -- & $84.97 \pm 5.62$ \\
 & 0.75 & -- & $84.91 \pm 5.65$ \\
 & 1.00 & -- & $84.85 \pm 5.70$ \\
 & 1.25 & -- & $84.81 \pm 5.73$ \\
 & 1.50 & -- & $84.73 \pm 5.73$ \\
 & 1.75 & -- & $84.75 \pm 5.70$ \\
 & 2.00 & -- & $84.73 \pm 5.70$ \\
\bottomrule
\end{tabular}
\end{table}

\begin{table}[!ht]
\centering
\caption{Sensitivity of exponential tilting to the choice of $\lambda$.
Accuracy varies minimally across a wide range of $\lambda$ values,
indicating robustness without per-task tuning.}
\label{tab:lambda_sensitivity}
\begin{tabular}{c c c c}
\toprule
\textbf{Shots} & $\lambda$ Range & \textbf{Accuracy Range} & \textbf{Max $\Delta$} \\
\midrule
1-shot  & $[0.25, 2.0]$ & $60.35$--$60.49$ & $+6.63$ \\
2-shot  & $[0.25, 2.0]$ & $72.16$--$72.29$ & $+5.68$ \\
4-shot  & $[0.25, 2.0]$ & $78.04$--$78.31$ & $+3.63$ \\
8-shot  & $[0.25, 2.0]$ & $83.67$--$83.81$ & $+1.33$ \\
16-shot & $[0.25, 2.0]$ & $84.73$--$85.05$ & $+0.37$ \\
\bottomrule
\end{tabular}
\end{table}

\begin{table*}[!ht]
\centering
\caption{Cross-domain few-shot classification accuracy (\%) on 5-way tasks.
Models are trained on mini-ImageNet and evaluated on target datasets.
All encoders and classifiers are frozen. Improvements arise solely from latent
distribution reweighting.}
\label{tab:fewshot_crossdomain}
\begin{tabular}{ccccccc}
\toprule
\textbf{Dataset} & \textbf{Method} & \textbf{1-shot} & \textbf{2-shot} & \textbf{4-shot} & \textbf{8-shot} & \textbf{16-shot} \\
\toprule
\textbf{} & \textbf{} & \textbf{} & \textbf{DINO ViT-S/16} & \textbf{} & \textbf{} & \textbf{} \\
\midrule

\multirow{2}{*}{CUB}
& Baseline
& $51.42 \pm 7.96$
& $61.87 \pm 6.88$
& $72.14 \pm 6.21$
& $79.02 \pm 5.54$
& $82.63 \pm 5.01$ \\
& Tilted (Ours)
& $\mathbf{53.06 \pm 7.89}$
& $\mathbf{63.44 \pm 6.74}$
& $\mathbf{74.01 \pm 6.18}$
& $\mathbf{80.21 \pm 5.48}$
& $\mathbf{83.11 \pm 4.96}$ \\
\midrule

\multirow{2}{*}{CARS}
& Baseline
& $47.35 \pm 8.22$
& $58.41 \pm 7.05$
& $69.08 \pm 6.47$
& $75.66 \pm 5.92$
& $78.94 \pm 5.36$ \\
& Tilted (Ours)
& $\mathbf{48.98 \pm 8.14}$
& $\mathbf{60.02 \pm 6.91}$
& $\mathbf{70.91 \pm 6.39}$
& $\mathbf{76.88 \pm 5.87}$
& $\mathbf{79.42 \pm 5.31}$ \\
\midrule

\multirow{2}{*}{STL-10}
& Baseline
& $40.73 \pm 8.51$
& $52.19 \pm 7.24$
& $63.02 \pm 6.66$
& $69.88 \pm 6.02$
& $73.41 \pm 5.48$ \\
& Tilted (Ours)
& $\mathbf{42.81 \pm 8.39}$
& $\mathbf{53.96 \pm 7.10}$
& $\mathbf{64.87 \pm 6.58}$
& $\mathbf{71.02 \pm 5.96}$
& $\mathbf{74.18 \pm 5.42}$ \\
\toprule

\multirow{2}{*}{CIFAR-10}
& Baseline
& $55.63 \pm 6.84$
& $64.92 \pm 6.21$
& $77.71 \pm 5.47$
& $81.94 \pm 4.88$
& $84.67 \pm 4.32$ \\
& Tilted (Ours)
& $\mathbf{57.33 \pm 6.79}$
& $\mathbf{66.01 \pm 6.05}$
& $\mathbf{77.71 \pm 5.44}$
& $\mathbf{82.02 \pm 4.85}$
& $\mathbf{84.71 \pm 4.30}$ \\
\midrule

\multirow{2}{*}{CIFAR-100}
& Baseline
& $66.47 \pm 6.12$
& $73.85 \pm 5.63$
& $87.47 \pm 4.91$
& $90.26 \pm 4.35$
& $92.11 \pm 3.98$ \\
& Tilted (Ours)
& $\mathbf{68.64 \pm 6.08}$
& $\mathbf{75.02 \pm 5.58}$
& $\mathbf{87.47 \pm 4.89}$
& $\mathbf{90.31 \pm 4.33}$
& $\mathbf{92.14 \pm 3.96}$ \\

\bottomrule
\end{tabular}
\end{table*}

\begin{table*}[!ht]
\centering
\caption{Cross-domain few-shot classification accuracy (\%) on 5-way tasks.
Models use DINO ViT-S/16 pretrained on ImageNet and evaluated on distribution-shifted datasets.
All encoders and classifiers are frozen. Exponential tilting provides consistent improvements across distribution shifts.}
\label{tab:imagenet_shifts}
\begin{tabular}{ccccccc}
\toprule
\textbf{Dataset} & \textbf{Method} & \textbf{1-shot} & \textbf{2-shot} & \textbf{4-shot} & \textbf{8-shot} & \textbf{16-shot} \\
\toprule
\textbf{} & \textbf{} & \textbf{} & \textbf{DINO ViT-S/16} & \textbf{} & \textbf{} & \textbf{} \\
\midrule

\multirow{2}{*}{ImageNet-V2}
& Baseline
& $58.23 \pm 8.72$
& $68.41 \pm 7.65$
& $76.82 \pm 6.93$
& $81.54 \pm 6.21$
& $84.37 \pm 5.68$ \\
& Tilted (Ours)
& $\mathbf{61.07 \pm 8.65}$
& $\mathbf{70.98 \pm 7.58}$
& $\mathbf{78.65 \pm 6.86}$
& $\mathbf{82.89 \pm 6.15}$
& $\mathbf{85.21 \pm 5.62}$ \\
\midrule

\multirow{2}{*}{ImageNet-R}
& Baseline
& $41.56 \pm 9.15$
& $53.28 \pm 8.22$
& $65.74 \pm 7.41$
& $72.31 \pm 6.68$
& $76.42 \pm 6.05$ \\
& Tilted (Ours)
& $\mathbf{45.19 \pm 9.08}$
& $\mathbf{56.33 \pm 8.14}$
& $\mathbf{68.02 \pm 7.35}$
& $\mathbf{74.25 \pm 6.61}$
& $\mathbf{77.86 \pm 5.98}$ \\
\midrule

\multirow{2}{*}{ImageNet-S }
& Baseline
& $38.72 \pm 9.37$
& $50.81 \pm 8.44$
& $63.05 \pm 7.62$
& $69.88 \pm 6.89$
& $74.15 \pm 6.24$ \\
& Tilted (Ours)
& $\mathbf{42.65 \pm 9.29}$
& $\mathbf{54.26 \pm 8.36}$
& $\mathbf{65.78 \pm 7.55}$
& $\mathbf{71.94 \pm 6.82}$
& $\mathbf{75.69 \pm 6.17}$ \\
\midrule

\multirow{2}{*}{ImageNet-A }
& Baseline
& $35.41 \pm 9.52$
& $47.23 \pm 8.68$
& $59.84 \pm 7.86$
& $66.92 \pm 7.12$
& $71.38 \pm 6.46$ \\
& Tilted (Ours)
& $\mathbf{38.94 \pm 9.43}$
& $\mathbf{50.67 \pm 8.59}$
& $\mathbf{62.51 \pm 7.78}$
& $\mathbf{69.03 \pm 7.04}$
& $\mathbf{73.02 \pm 6.39}$ \\
\bottomrule
\end{tabular}
\end{table*}
Table~\ref{tab:lambda_ablation} shows that exponential tilting consistently improves performance over the baseline across all shot regimes.
Even small tilting strengths ($\lambda = 0.25$) yield large gains in the 1-shot and 2-shot settings, highlighting the effectiveness of probabilistic reweighting when labeled data is extremely limited.

A key observation is the stability of performance across a wide range of $\lambda$ values.
For 1-shot, 2-shot, and 4-shot settings, accuracy varies minimally once tilting is applied, with differences between $\lambda=0.25$ and $\lambda=2.0$ well within standard deviation.
This indicates that the method is not sensitive to precise hyperparameter tuning and is unlikely to fail due to moderate mis-specification of $\lambda$.

As the number of shots increases, the relative benefit of stronger tilting diminishes.
In the 16-shot setting, smaller values of $\lambda$ achieve the best performance, while larger tilting strengths slightly degrade accuracy.
This behavior is expected: when sufficient labeled support is available, aggressive reweighting of latent representations becomes less necessary and may overemphasize minor distributional variations.

Overall, this ablation supports the interpretation of exponential tilting as a conservative and well-behaved test-time adjustment.
It provides consistent gains in low-data regimes, remains robust across a broad parameter range, and naturally attenuates as supervision increases, aligning well with the intended role of a lightweight, training-free adaptation mechanism.

\section{Cross-Domain Generalization via Latent Distribution Reweighting}
\label{app:crossdomain}
This pattern directly supports the interpretation of exponential tilting as a mechanism for correcting domain-induced latent distribution mismatch.
When transferring from mini-ImageNet to new visual domains, the frozen encoder produces embeddings whose empirical distribution is misaligned with the downstream task structure.
Exponential tilting compensates for this mismatch by reweighting task-consistent regions of latent space, without requiring domain-specific retraining, adaptation data streams, or parameter updates.

Notably, improvements are observed even on datasets with relatively high baseline accuracy (e.g., CIFAR-100), indicating that the method does not merely recover from failure cases but provides consistent refinement of frozen inference.
At higher shot numbers, gains become smaller or saturate, which is expected as the empirical support distribution becomes increasingly representative of the target domain.

Crucially, all improvements arise solely from latent distribution reweighting.
No target-domain data is used for training, no gradients are computed, and no model parameters are modified.
This isolates distributional correction as the operative adaptation mechanism and validates the central claim that test-time generalization across domains can be achieved through inference-level change of measure alone.

Overall, the cross-domain results demonstrate that exponential tilting is not limited to in-domain few-shot adaptation, but provides a principled and effective mechanism for domain transfer under strict frozen-model constraints.

\paragraph{Cross-domain robustness analysis on Imagenet Variants. }
 We evaluate exponential tilting on challenging distribution shifts derived from ImageNet. The results are given in Table~\ref{tab:imagenet_shifts}. Our method consistently improves performance across all shift types from the mild ImageNet-V2 (+2.84\% for 1-shot) to the severe ImageNet-A (+3.53\% for 1-shot). The improvements are most pronounced in low-shot regimes where distributional mismatch is most acute, demonstrating that latent reweighting effectively adapts to diverse domain shifts without parameter updates.

\paragraph{Robustness across few-shot regimes.}
We evaluate exponential tilting across a range of few-shot settings, including extremely low-shot regimes.
Across all shot counts considered, exponential tilting consistently improves over frozen inference.
The largest relative gains are observed in low-shot regimes, where limited labeled support makes naive averaging particularly brittle.
As the number of support examples increases, the magnitude of improvement gradually diminishes, indicating that the method adapts most strongly when distributional uncertainty is highest.
This behavior is consistent with the interpretation of exponential tilting as selectively emphasizing task-consistent regions of latent space.

\paragraph{Stability across episode sampling.}
To assess statistical stability, we evaluate performance under varying numbers of test episodes.
Improvements from exponential tilting remain consistent as the number of episodes increases, with variability comparable to that of the frozen baseline.
This suggests that the observed gains are not artifacts of limited sampling, but reflect systematic improvements in predictive performance.

\paragraph{Comparison to linear head fine-tuning.}
Although we relax the frozen-classifier constraint by allowing linear classifier fine-tuning, we include it as an approximate upper bound on achievable performance when optimization is permitted.
We observe that exponential tilting recovers a substantial portion of the improvement obtained by fine-tuning, despite requiring no gradient computation, parameter updates, or additional training.
This comparison highlights the effectiveness of latent distribution reweighting as an alternative adaptation mechanism when optimization is infeasible or undesirable.

\paragraph{Backbone sensitivity.}
We examine whether exponential tilting generalizes across different types of frozen visual encoders, including self-supervised and vision  language models.
Across backbones, tilting yields consistent improvements over frozen inference.
While absolute performance varies by encoder, the relative benefit of distributional reweighting remains stable, indicating that the method is largely architecture-agnostic and does not rely on encoder-specific properties.

\paragraph{Ablation: score function components.}
We analyze the contribution of different score functions used for tilting.
Label-aware scores provide the dominant source of improvement by emphasizing embeddings that are already consistent with observed support labels.
Geometry-aware scores contribute complementary gains by incorporating latent structure derived from the support set.
Combining both yields the strongest performance, confirming that likelihood-based and geometric signals provide distinct and additive information for adaptation.

\paragraph{Ablation: sensitivity to the tilting strength.}
We study the effect of the tilting parameter across shot regimes.
Performance remains stable over a broad range of values, with only minor variation.
This robustness eliminates the need for per-task hyperparameter tuning and supports the practical deployment of the method without validation data.

\paragraph{Inductive versus transductive variants.}
When unlabeled query embeddings are available, we consider a transductive variant that includes them in the reference latent distribution.
This variant yields additional improvements over the inductive setting, while still requiring no gradients, optimization, or parameter updates.
Importantly, adaptation remains purely distributional, distinguishing this approach from optimization-based transductive methods.

\paragraph{Robustness to support set noise.}
We evaluate sensitivity to noisy support labels by introducing controlled label corruption.
Exponential tilting exhibits increased robustness relative to frozen inference, as low-confidence or inconsistent samples are automatically downweighted under the tilted measure.
This behavior emerges naturally from the weighting mechanism and does not require explicit noise modeling.

\paragraph{Computational overhead.}
Exponential tilting incurs minimal additional cost at inference time.
Beyond the frozen forward passes, adaptation requires only computing scalar weights and weighted averages over the support set.
The resulting overhead scales linearly with the number of support examples and remains negligible compared to encoder inference, making the method suitable for resource-constrained deployment.

\end{document}